\documentclass{article}
\usepackage{fullpage}

\usepackage{natbib}
\usepackage{mathtools}
\usepackage{amssymb,amsmath,amsthm}
\usepackage{color}
\usepackage{dsfont}
\usepackage{algpseudocode}
\usepackage{natbib}
\usepackage{bm}

\def\vec#1{{#1}}

\def\Var{\mathrm{Var}} 
\def\Cov{\mathrm{Cov}} 
\def\E{\mathbb{E}}

\def\a{\vec{a}}

\def\prior{\mathcal{P}}

\newtheorem{theorem}{Theorem}
\newtheorem{theorem*}{Theorem}

\newtheorem{lemma}{Lemma}

\def\newpi{PI++\ }

\begin{document}

\title{Off-Policy Evaluation of Slate Policies under Bayes Risk}

\author{Nikos Vlassis \and Fernando Amat Gil \and  Ashok Chandrashekar }

\date{Netflix Research, CA, USA }

\maketitle

\begin{abstract}
We study the problem of off-policy evaluation for slate bandits, for the typical case in which the logging policy factorizes over the slots of the slate. We slightly depart from the existing literature by taking Bayes risk as the criterion by which to evaluate estimators, and we analyze the family of `additive' estimators that includes the pseudoinverse (PI) estimator of Swaminathan et al.\ (2017). Using a control variate approach, we identify a new estimator in this family that is guaranteed to have lower risk than PI in the above class of problems. In particular, we show that the risk improvement over PI grows linearly with the number of slots, and linearly with the gap between the arithmetic and the harmonic mean of a set of slot-level divergences between the logging and the target policy. In the typical case of a uniform logging policy and a deterministic target policy, each divergence corresponds to slot size, showing that maximal gains can be obtained for slate problems with diverse numbers of actions per slot.
\end{abstract}

\section{Introduction}
\label{sec:intro}
Online services (news, music and video streaming, social media, e-commerce, app stores, etc) typically optimize the content on the homepage of a customer via interactive machine learning. These landing pages can be viewed as high dimensional slates with each slot on the slate having multiple candidate items (actions). For instance, a personalized news service may have separate slots for local news, sports, politics, etc with multiple news items for each slot to select from. In many cases, the number of slot actions may be in the hundreds or thousands, thus making the overall slate a high dimensional construct that can require very large amounts of user engagement data for optimization. The online service may have a robust A/B experimentation system to iteratively optimize the experience, but, given the combinatorial action space, the latter may entail large regret. 

In these cases it becomes vital to have an efficient \emph{off-policy evaluation}(OPE) framework to make progress. In OPE we use collected interaction data from a (stochastic) logging policy in order to evaluate a new candidate target policy, thus reducing the need for A/B experiments and their associated cost
\citep{dudik2011doubly, bottou2013counterfactual,
thomas2016data, swaminathan2017nips, 
kallus2018robust, 
farajtabar2018icml,
joachims2018iclr,Gilotte_2018, su2020doubly}. 
Some examples of the use of OPE in real-world slate optimization problems include
\citet{shalom_2016} who apply OPE to a large scale real world recommender system that handles purchases from tens of millions of users, 
\citet{Hill_2017} who present a method to optimize a message that is composed of separated widgets (slots) such as title text, offer details, image, etc, and
\citet{McInerney_2020} who apply OPE to a slate problem involving sequences of items such as music playlists. 

Our work is motivated by real world slate problems in which slots may differ a lot in terms of the number of available actions. Using the pseudoinverse (PI) estimator of \citet{swaminathan2017nips} as our base estimator, and taking a control variate approach, we identify a new PI-like estimator that is guaranteed (under certain conditions) to have lower Bayes risk than PI. The new estimator is defined through a \emph{weighted} sum over slots, with the weights optimized to reflect differences in slot action counts. We provide theoretical support and empirical evidence that the risk gains over PI can be substantial.

\section{Notation and setup}
\label{notation}

We consider contextual slate bandits where each slate has $K$ slots. We will use $[K]$ to denote the set $\{1,2,\ldots,K\}$.
The number of available actions of slot $k$ is denoted by $d_k$.
We will use $A$ to denote a random slate,  and $A^k$ to refer to its $k$-th slot action ($k \in [K]$). We will use $X$ for a random context, and $R$ for random slate-level reward. We will use corresponding lower case $(a,a^k,x,r)$ to refer to realized values. 
  
We assume a Bernoulli reward model at the slate level, denoted by $p(a,x)$. (This is a simplifying assumption, mainly for presentation purposes. Our main theorem holds for continuous rewards too.) We assume i.i.d.\ data $\{x_i, a_i, r_i\}_{i=1}^n$, where the slates $\a_i$ are obtained using a stochastic logging policy $\mu(A|X)$ that factorizes over the slot actions conditioned on context
\[
\mu(A | X) = \prod_{k=1}^K \mu_k(A^k | X) 
\]
(which is a typical scenario in practical applications that motivated our work).
The rewards $r_i$ are draws from the corresponding slate-level Bernoulli, $r_i \sim p(a,x)$. We want to use these data to estimate the value $v_\pi \equiv \E_\pi [R]$ of a target policy $\pi$.
We will assume absolute continuity of $\pi$ with
respect to $\mu$, that is $\mu(a|x) > 0$ when $\pi(a|x) > 0$.

To simplify notation, we will throughout use $\E[\cdot]$, $\Var(\cdot)$, and $\Cov(\cdot,\cdot)$ (with no subscripts) to refer to population moments w.r.t $\mu$. In all other cases, moments will be subscripted accordingly (e.g., $\E_\pi[\cdot]$ will denote expectation w.r.t $\pi$).

\section{A general family of estimators}
We will consider estimators of the general form
\begin{align}
    \label{generalestimator}
    t = \frac{1}{n} \sum_{i=1}^n \big(r_i \, g_i - f_i \big) \, ,
\end{align}
where $g_i \equiv g(a_i,x_i)$ and $f_i \equiv f(a_i,x_i)$ for  
some functions $g$ and $f$ that depend on $\mu$ and $\pi$ (and may also involve extra free parameters).
Note that $a_i$ and $x_i$ are random draws (from $\mu$), whereas the functions $g$ and $f$ are assumed fixed.  
The classical inverse propensity scoring (IPS) estimator \citep{horvitz1952generalization} is an unbiased estimator from the family \eqref{generalestimator}, corresponding to the choice
\begin{align}
    g_i = \frac{\pi(a_i | x_i)}{\mu(a_i | x_i)} \, ,
    \qquad 
    f_i = 0 \, .
\end{align}
When $a_i$ is a slate, IPS is known to suffer from very high variance \citep{swaminathan2017nips}.

To mitigate the variance explosion in slate problems, we will restrict attention to functions $g$ and $f$ that decompose additively over the slots of the slate, e.g.,
\begin{align}    
    \label{1way}
    g_i = \lambda + \sum_{k=1}^K w_k \, 
    \frac{\pi(a_i^k | x_i)}{\mu_k(a_i^k | x_i)} \, ,
\end{align}
where $\lambda \in \mathbb{R}$ and $\{w_k \in \mathbb{R}\}_{k=1}^K$ are free parameters. Under a factored logging policy $\mu$, this family of estimators includes the `pseudoinverse' (PI) estimator of \citet{swaminathan2017nips}, which is obtained for
\begin{align}   
\lambda=1-K \, , \qquad 
w_k=1 \ \ \forall k \in [K] \, , \qquad
f_i=0 \, .
\end{align}

\section{Bayes risk}

We will analyze the above family of estimators using   \emph{Bayes risk}. The Bayes risk $\rho(t)$ of an estimator $t$ is defined as the expected mean squared error (MSE) of $t$, where the expectation is taken with respect to a prior $\prior$ over the true model primitives (here the slate-level Bernoulli rates $p(a,x)$). For example, the Bayes risk of an unbiased estimator $t$ in \eqref{generalestimator} is
\begin{align}
    \label{bayesriskgeneral}
    \rho(t) =
        \E_{\prior} \big[\Var(t) \big] \, .
\end{align}

The use of Bayes risk is motivated by practical applications in which the analyst often possesses some knowledge about the expected rewards of the problem (e.g., from past experiments). As we will see next, when we restrict our analysis to a certain subfamily of \eqref{generalestimator}, the Bayes risk will depend only on the prior \emph{mean} of $p(a,x)$ under $\prior$. Knowledge of the latter is often readily available in applications.

To facilitate the analysis of risk for estimators in the family \eqref{generalestimator} we will restrict attention to estimators in which their $f$ component (which can be regarded as a control variate over the $g$ component) confers zero additional bias. Subject to that condition, we get a compact analytical bound for the variance $\Var(t)$ as we show next. 

Since we work with i.i.d.\ data, we can compute population moments of $t$ in \eqref{generalestimator} by analyzing a single term. 
Let $R$ be the random slate-level reward, as defined in Section \ref{notation}.
To simplify the derivations, we additionally define the following random variables: 
\begin{itemize}
    \item 
    $G \equiv g(A,X)$ is the application of the function $g$ on the random slate-context pair $(A,X) \sim \mu$, and $F \equiv f(A,X)$ is defined analogously.
    \item
    $P \equiv p(A,X)$ is random slate-level Bernoulli rate, where, as in the definition of $G,F$ above, the randomness is due to $(A,X) \sim \mu$.
\end{itemize}
 The random quantity of interest is then 
\begin{align}
    T = R G  - F  \qquad 
    \mbox{s.t.} \qquad  \E[ F ] = 0\, .
\end{align}
For the variance of $t$ we have
$\Var(t) = \frac{1}{n} \Var(T)$, and  
we can compute the latter as follows:

%
\begin{lemma}
Subject to $\E[F]=0$ we have
\begin{equation}
\label{VarTbound}
\mathrm{Var}[T] = 
    \E [ P G^2 ] 
     - \E [P G]^2
    + \E [ F^2 ] 
    - 2 \, \E [ P G F ] . \hfill
\end{equation}

\end{lemma}
%
\begin{proof}
We have
\begin{align}
\label{varT}
\Var(T) = \Var(RG) + \Var(F) - 2 \, \Cov(RG,F) . \quad
\end{align}
Using the law of total variance, the first term in \eqref{varT} can be computed as follows:
\begin{align}
    & \Var (RG)
    =
    \Var(PG)
    +
    \E \big[ P (1-P) G^2 \big]  
    \\
    & \qquad = 
    \E [ P^2 G^2 ] 
    - 
     \E [P G]^2
    +
    \E [ P (1-P) G^2 ]  
    \\ 
    & \qquad =
    \E [ P G^2 ] 
    -
    \E [P G]^2 \, .
\end{align}
The last two terms in \eqref{varT} can be computed as follows:
\begin{align}
\Var(F) &= 
    \E [ F^2 ] - {\underbrace{\E[F]}_{=0}}^2
    = \E [ F^2 ] \label{varF} \\
\Cov(RG,F) &=  
    \E [ R G F ]  - 
    \E [ R G ] \, \underbrace{\E[F]}_{=0} = \E [ P G F ] \label{covRGF}
\end{align}
where in the last equation we used the tower rule for expectations.
The result \eqref{VarTbound} then follows.
\end{proof}

A few observations follow from \eqref{VarTbound}. First, we note that we can bound the variance $\Var(T)$ by a quantity (dropping the second term) in which the Bernoulli rates $P \equiv p(A,X)$ appear only linearly. This implies that, for unbiased estimators, the Bayes risk \eqref{bayesriskgeneral} can be bounded by a quantity in which the rates appear only by their expectation under $\prior$. That prompts ideas for directly optimizing such a bound of Bayes risk under decompositions such as \eqref{1way} (but we do not pursue this route here).
A second observation from \eqref{VarTbound} is that the interaction between $G$ and $F$ (the last term in \eqref{VarTbound}) involves again the Bernoulli rates $P \equiv p(A,X)$ only linearly (which is not a surprise, since that correponds to a covariance term). 

The above has motivated the following approach to analyze the family of estimators \eqref{generalestimator}: We assume that the $g$ component of the estimator is fixed and we are only free to vary the $f$ component (a control variate approach). This approach has the advantage that it `works' regardless of whether the base estimator (corresponding to $f=0$) is biased or not, as we only care about the reduction of Bayes risk through $f$.
Let us denote by $\bar P$ the mean Bernoulli rate under prior $\prior$ (we assume for simplicity that $\bar P$ is independent of $A,X$). Then, using \eqref{VarTbound} and subject to $\E[F]=0$, the Bayes risk as a function of $F$ reads 
\begin{align}
    \label{bayesriskbound}
    n \, \rho(t) = 
\E [ F^2 ] 
    - 2 \bar P \, \E [  G F ]  
    + \mbox{const.} \hfill
\end{align}
In the next section we will use this simplified form of Bayes risk to analyze estimators in the family \eqref{generalestimator}.

\section{Additive estimators}

We now return to the `additive' decomposition \eqref{1way} for parametrizing the random $G$ and $F$ in \eqref{bayesriskbound}. 
To simplify derivations, we define the following slot-level `importance weight' random variables:
\begin{align}
\label{Yk}
Y_k \equiv \frac{\pi(A^k | X)}{\mu_k(A^k | X)} \qquad \forall k \in [K] \, .
\end{align}
Note that $\E [Y_k Z | X] = \E_\pi [Z | X]$ for any $Z$ (note the change of measure; the first expectation is over $\mu$, the second is over $\pi$). It  follows that
\begin{align}
\E [Y_k | X] &= 1 \\
\E [Y_k^2 | X] &= \E_\pi [Y_k | X] \\
\Var(Y_k | X) &= \E_\pi [Y_k | X] - 1 \\
\E [Y_k ] &= \E \big[\E [Y_k | X]\big] = 1 \label{expYk} \\
\Var(Y_k) &= \E \big[ \Var(Y_k | X) \big] + \\
& \quad \Var \big( \E [Y_k | X] \big) = \E_\pi [Y_k] - 1  \label{varYk} \\
\Cov(Y_k,Y_j) &= \E \big[ \Cov(Y_k,Y_j | X) \big] + \\ 
    & \hspace*{-25pt} \Cov \big( \E[Y_k|X], \E[Y_j|X] \big) = 0 \quad (k \neq j) \quad \label{covYkYj}
\end{align}
where in \eqref{varYk} we used the law of total variance and the fact that $\Var \big( \E [Y_k | X] \big) = 0$, and in \eqref{covYkYj} we used the law of total covariance and the fact that both terms are zero when $\mu(A|X)$ factorizes over slots.
We will henceforth more compactly write 
\begin{align}
\alpha_k \equiv \Var(Y_k) = \E_\pi [Y_k] - 1 \, . \label{alpha}
\end{align}
It follows from \eqref{varYk} that the $\alpha_k$ quantities in \eqref{alpha} define a set of `slot-level' $\chi^2$-divergences between $\mu$ and $\pi$.
As an example, when $\mu$ is uniform and $\pi$ is deterministic we have $\alpha_k = d_k - 1$, where $d_k$ is the number of available actions of the $k$-th slot.
Our main result next is formulated in terms of this set of slot-level divergences between $\mu$ and $\pi$.

We can now use the definition of $Y_k$ from \eqref{Yk}, and the decomposition \eqref{1way}, to express the random $G$ and $F$ in \eqref{bayesriskbound}. We restrict attention to a subfamily of estimators in \eqref{generalestimator} in which the $g$ component is given by an \emph{unweighted} sum (as in the PI estimator) and the $f$ component is given by a \emph{weighted} sum with a zero constant term:
\begin{align}
G &= \lambda + \sum_{k=1}^K Y_k \, ,
\qquad \lambda \in \mathbb{R} \label{G1way} \\
F &= \sum_{k=1}^K w_k \, Y_k  \, ,
\qquad w_k \in \mathbb{R} 
 \ \ \forall k \in [K] \label{F1way}
 \, .
\end{align}
As we will see next, this subfamily contains an estimator that is never worse (and often is much better) than PI in terms of Bayes risk. In the light of \eqref{bayesriskbound}, that implies the existence of a decomposition for $F$ such that
$\E [ F^2 ] - 2 \bar P \E[G F] \leq 0$.

We first show that, for the choice of $G$ and $F$ from \eqref{G1way} and \eqref{F1way}, and subject to $\E[F]=0$, the expectations that appear in the Bayes risk \eqref{bayesriskbound} admit compact analytical solutions. (We will next write $\sum_k$ for $\sum_{k=1}^K$.)

\begin{lemma}
For $G$ and $F$ according to \eqref{G1way} and \eqref{F1way}, and subject to $\E[F]=0$, we have
\begin{align}
\E [F^2] &= \sum_k w_k^2 \, \alpha_k      \label{EF2}
\\
\E [ G F ] &= \sum_k w_k \, \alpha_k     \label{EGF}
\end{align}
with $\alpha_k$ defined in \eqref{alpha}.
\end{lemma}
%
\begin{proof}
For the term $\E[F^2]$ we can write
\begin{align}
\E [F^2]
&= \Var(F) +  {\underbrace{\E [F]}_{=0} }^2 \\
& =  \sum_k w_k^2 \, \Var(Y_k)  + 
    \sum_{k, j \neq k} \underbrace{\Cov ( w_k Y_k, w_j Y_j )}_{=0} \\
    &= 
    \sum_k w_k^2 \, \alpha_k \, ,
\end{align}
where in the second equation we used \eqref{covYkYj}.
Analogously, for the term $\E[GF]$ we have
\begin{align}
& \E [ G F ]  
= \Cov ( G, F ) +  \E [ G ] \, \underbrace{\E [ F ]}_{=0} \\
&=  \sum_k \Cov ( Y_k, w_k Y_k ) + 
    \sum_{k, j \neq k} \underbrace{\Cov ( Y_k, w_j Y_j )}_{=0}   \\
& =  \sum_k w_k \, \Var ( Y_k )  = 
    \sum_k w_k \, \alpha_k \, .
\end{align}
\end{proof}
Now we have all ingredients in place to establish our main result:
\begin{theorem}
Under a factorized logging policy, there exists an estimator in the family \eqref{generalestimator} that is guaranteed to have lower Bayes risk than PI. In particular, the improvement in risk is
\begin{align}
    \label{mainresult}
  \bar P^2 \, K \, (M - H) \ge 0
 \, ,
\end{align}
where $\bar P$ is the mean Bernoulli rate under prior $\prior$, $K$ is the number of slots,
$M$ is the \emph{arithmetic} mean of the set of divergences $\{\alpha_k\}_{k=1}^K$ defined in \eqref{alpha},
and $H$ is the \emph{harmonic} mean of that set.
\end{theorem}
\begin{proof}
First note that using \eqref{expYk} we get 
\begin{align}
\E[F] &= \sum_k w_k \, \E[Y_k] = \sum_k w_k \,.
\end{align}
Plugging $\E[F^2]$ and $\E[GF]$ from \eqref{EF2} and \eqref{EGF} into \eqref{bayesriskbound}, we get the following linearly constrained quadratic program in the free parameters $w_k$:
\begin{align}
\min_{w_k} &\ \  \sum_k w_k^2 \, \alpha_k 
     - 
        2 \bar P \sum_k w_k \, \alpha_k  \label{objective2} \\
\mbox{s.t.}  &\ \ \sum_k w_k = 0  \, , \label{wconstraint}
\end{align}
whose solution can be easily verified to be
\begin{align}
    \label{wopt}
    w_k^* = \bar P \, \Big( 1 - \frac{H}{\alpha_k} \Big) \, ,
\end{align}
where $H$ is the \emph{harmonic} mean of the set $\{\alpha_k\}_{k=1}^K$.
For the optimal $w_k^*$ from \eqref{wopt}
the objective \eqref{objective2} reads
\begin{align}
\label{eqMH}
  - \bar P^2 \, K \, (M - H) 
 \, ,
\end{align}
where $M$ is the \emph{arithmetic} mean of the $\{\alpha_k\}_{k=1}^K$.
The inequality in \eqref{mainresult} follows from the classical HM-GM-AM inequalities.
\end{proof}

Let us briefly discuss the above result. Eq.~\eqref{mainresult}
shows the amount of risk reduction we can achieve over the PI estimator when using a weighted PI-like control variate.
How much of an improvement is this? To gain some insight, let us consider the case in which $\mu$ is uniform and $\pi$ is deterministic, in which case $\alpha_k$ corresponds to slot cardinality (see \eqref{alpha}). When all slots have the same cardinality, the improvement is zero, because the arithmetic and harmonic means coincide. In that case the control variate is not helping. On the other hand, if the slot cardinalities differ, the improvement can be significant, since the harmonic mean of a set of numbers can be much smaller than their arithmetic mean. Hence, our result shows that the new estimator can obtain maximal gains when the slots differ a lot in terms of action counts, which is a frequently occurring pattern in real world applications. In a sense, our approach corrects for the `symmetry' of the PI estimator in treating all slots equally via an unweighted sum in \eqref{G1way}. In the rest of the paper we will be using the name \newpi for the derived estimator. 

Note from \eqref{wopt} that the optimal estimator requires knowledge of the prior mean $\bar P$. 
If we instead use a different prior mean $P'$ than the true $\bar P$, the objective \eqref{objective2} reads
\begin{align}
    \label{improvementwrongP}
  P' (P' - 2 \bar P) \, K \, (M - H) 
 \, .
\end{align}
In the experiments section we will demonstrate empirically the effect of using a wrong prior.

\section{Experiments}
We employ simulations to corroborate the theoretical guarantees and further illustrate the characteristics of the proposed \newpi estimator.
We run ablation experiments under various settings that are relevant in different practical scenarios for slate OPE. The simulations are limited to the non-contextual bandit setting for simplicity, since the theory and its predictions for performance improvement over the PI estimator apply equally well to contextual and non-contextual bandits. This way we focus on providing guidance on problem settings in which the proposed \newpi estimator is expected to outperform the state of the art PI estimator under a factorized logging policy.

\subsection{Description of the simulations}
In the first set of simulations we assume an elementwise additive model for the Bernoulli rates
\begin{align}
     p(a) = \sum_{k=1}^K \phi_k(a^k)  \label{eq:Psum}
 \end{align}
for some functions $\phi_k$.
This is the setting in which PI under a factorized $\mu$ is known to be unbiased \citep{swaminathan2017nips}.
In Section \ref{sec:pairwise} we present results assuming a pairwise (across slots) additive model for the Bernoulli rates
 \begin{align}
 \label{pairwise}
     p(a) = \sum_k
     \sum_{j > k}  \phi_{kj}(a^k,a^j) \, .
 \end{align}
The main parameters of a simulated instance are the sample size $N$, the number of slots $K$, and the slot action cardinality tuple $D=[d_1,...,d_K]$ where $d_i$ is the number of available actions of the $i$-th slot. 
For our proposed estimator, we also need to specify the prior $P'$. 
In each simulation experiment, we first generate $T$ random reward tensors $p(a) \in [0,1]^{D}$ using \eqref{eq:Psum}, 
where the $\phi_k(\cdot)$ are drawn from a Gaussian distribution  $\mathcal{N}(\bar P / K, 0.1 \bar P / K)$.
For each sampled tensor, we generate a dataset of random slates and Bernoulli rewards, and use this dataset to evaluate a deterministic policy $\pi(a)=\mathbb{I}(a = [0,...,0])$ (wlog) with each of the candidate estimators. Since we know the ground truth, we can compute the bias and the variance of each estimator and from those its Bayes risk. In all results we report average $N*$MSE over the $T$ tensors.

For each experiment presented in this section, unless otherwise noted, we run $500$ simulations with $N=1e7$ for each sampled reward tensor. The number of drawn tensors is set to $T=200$, and we set $\bar P = 0.25$.

\begin{center}
\begin{figure*}
  \includegraphics[width=3.25in]{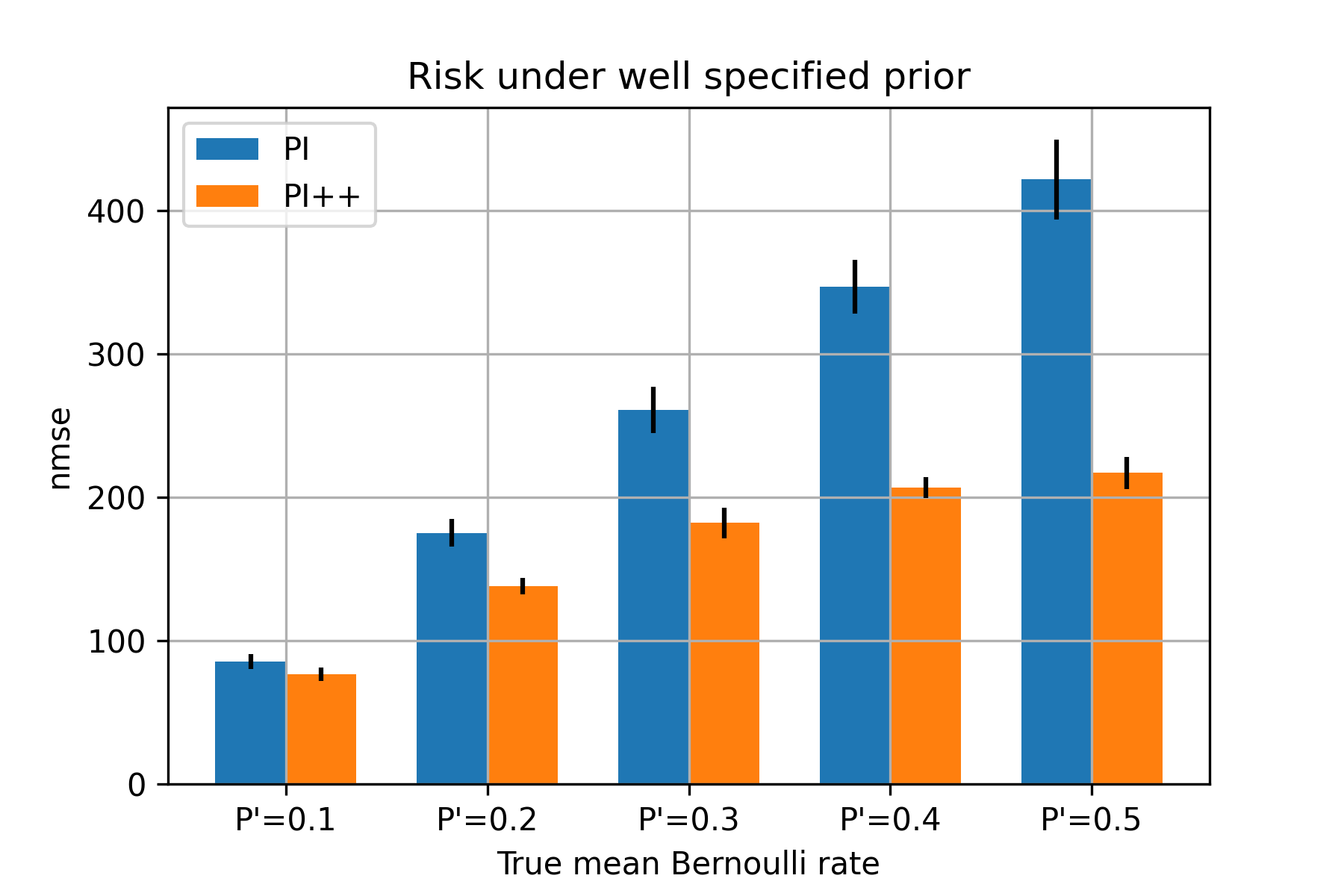}
  \includegraphics[width=3.25in]{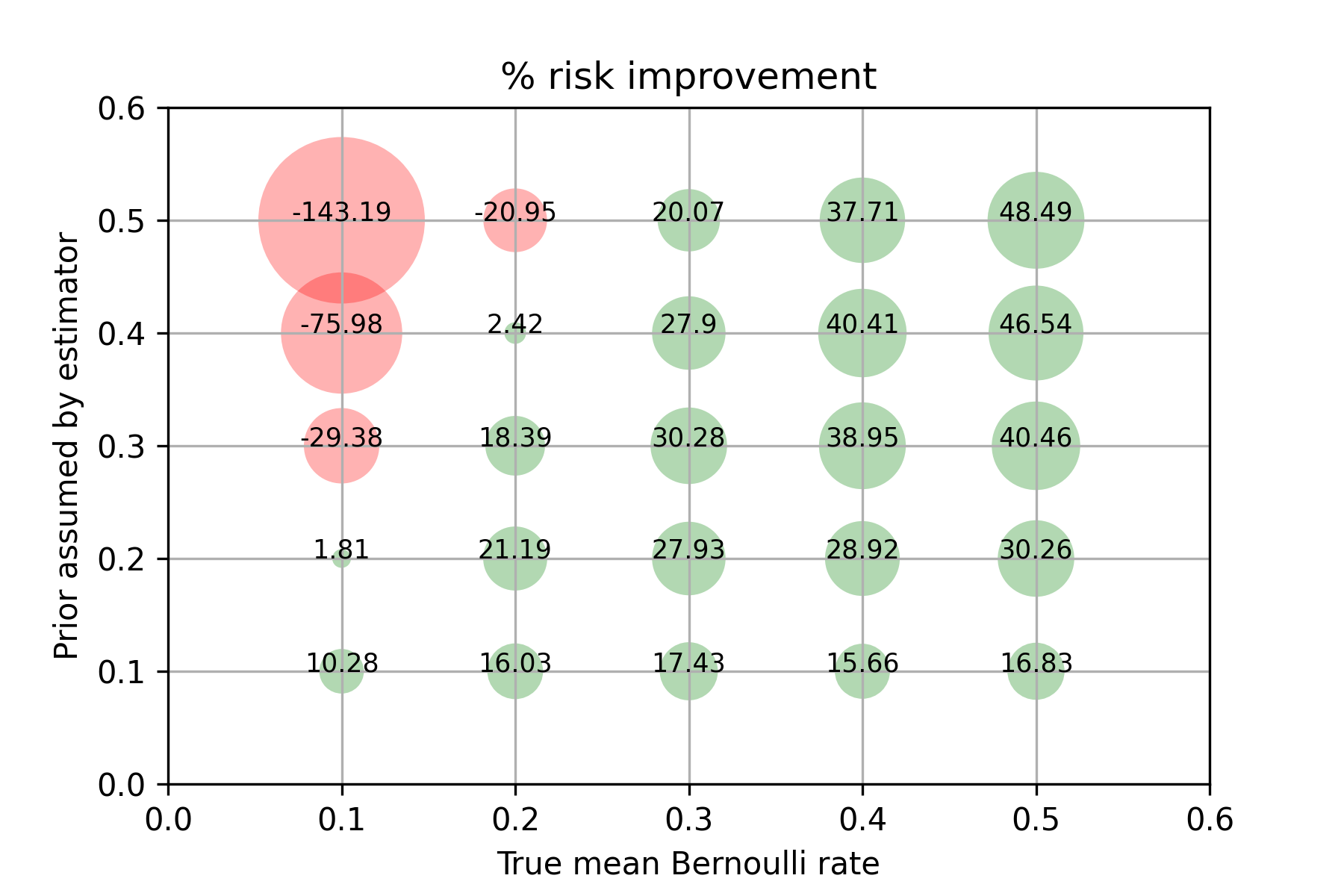}
  \caption{The impact of the assumed slate-level prior mean Bernoulli rate $P'$, on the performance of the \newpi estimator. The plot on the left shows the scenario where $P' = \bar P$. The figure on the right shows the impact of misspecification of the prior in comparison to the true bernoulli  rates. The simulation parameters corresponding to these results are: $T= 50, K = 3, D = [3, 50, 800]$ and $N= 1e7$.}  
  \label{fig:prior}
\end{figure*}
\end{center}

\subsection {Specification of the prior}
One of the salient aspects of  approaching estimator design from a Bayes risk perspective is the assumption (as part of the estimator) of a prior on the expected reward for the slates. In practice this is typically not a major concern, as the analyst working on an application can typically determine the reward probabilities through domain knowledge. In this experiment we investigate the impact of the assumed prior mean Bernoulli rate $P'$ w.r.t to the true mean $\bar P$. In this simulation, we fix $K = 3$, and action count tuple $D = [3, 50, 800]$. We then explore a grid of values for the true $\bar P$ and the $P'$ assumed by the estimator. We draw 50 tensors and conduct 1000 simulations per tensor. Each dataset per simulation consists of $N=1e7$ slates.

The results are shown in Figure \ref{fig:prior}. The plot on the left depicts the case in which the assumed prior accurately reflects the true prior, i.e., $P' =\bar P$. We observe that the improvement in risk of the \newpi estimator w.r.t the PI estimator increases with increasing $P'$ (as predicted by the analysis of the estimator in equation \eqref{mainresult}). From the plot on the right in Figure \ref{fig:prior} we see that for a majority of cases, even with a misspecified prior $P'$, \newpi achieves robust risk improvements over the PI estimator. In particular, as long as $P' \leq 2\bar P$, \newpi improves upon the PI estimator with the gains reaching their maximum when $P'$ approaches $\bar P$. This is in perfect agreement with the analytical result in \eqref{improvementwrongP}.

\subsection{Number of actions per slot}
In this experiment we investigate the effect of action set cardinalities for the slots. Large action sets pose a major challenge to slate optimization and off-policy evaluation, and hence it is critical to understand the behavior of the estimators in dealing with large action sets. In this simulation we use slates with two slots ($K=2$). For each slot we define the number of available actions by picking from the set $\{2, 10, 100, 1000\}$, and run the simulations as described previously. As in the previous experiment, we use sample size $N=1e7$, we generate $T=50$ tensors, and run 1000 simulations per tensor.

Figure \ref{fig:D1D2} shows the results of the simulations. Each point on the grid reflects a relative percentage improvement in risk achieved by the proposed \newpi estimator when compared to the original PI estimator. From the figure it is clear that the gains are robust when the number of actions differ between slots (nondiagonal points on the grid). The two estimators are largely equivalent when the two slots have the same number of actions, as predicted by a close reading of eq.~\eqref{eqMH} (see also the discussion after eq.~\eqref{eqMH}).

\begin{figure}
\begin{center}
  \includegraphics[width=3.4in]{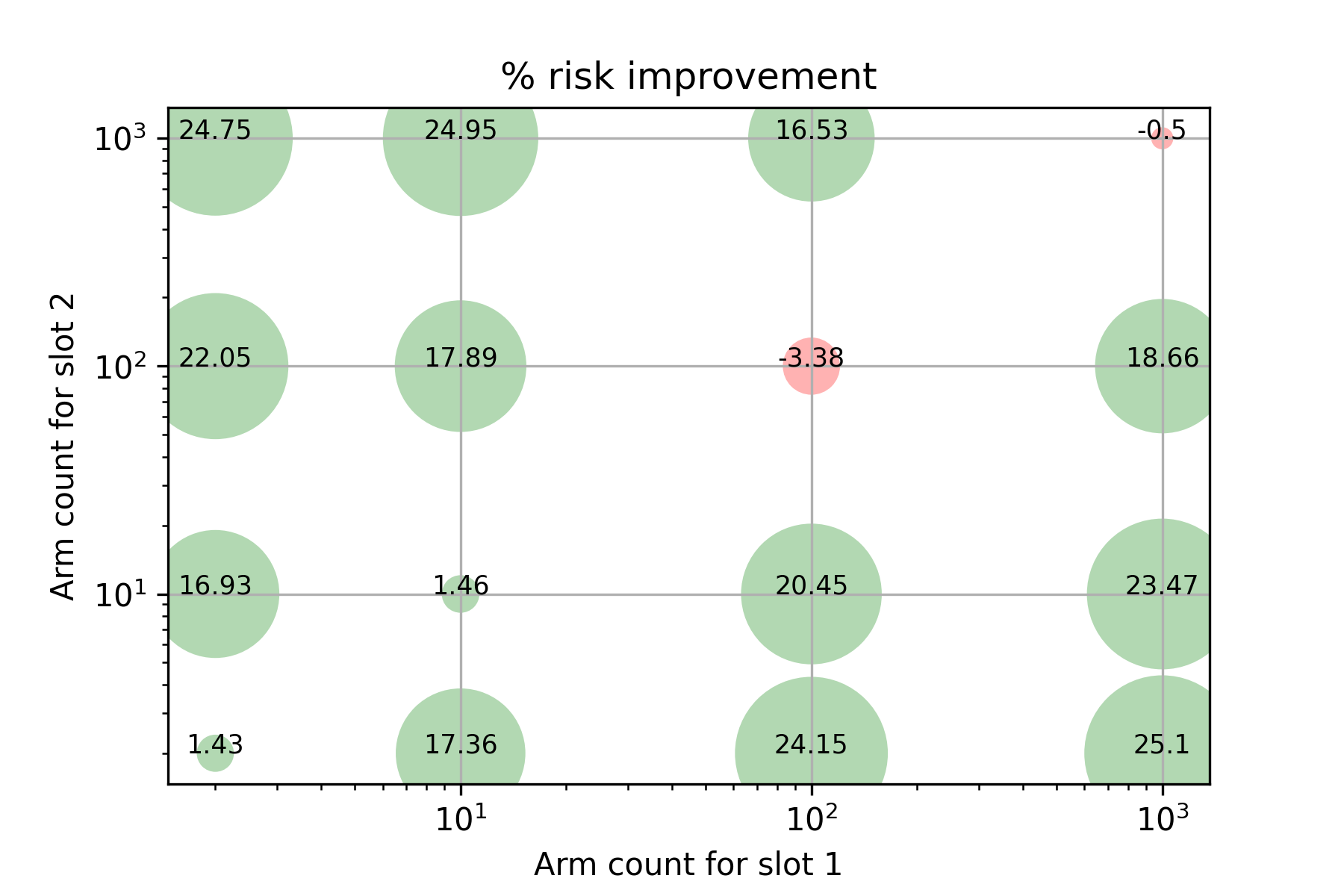}
  \caption{The behavior of the \newpi estimator in relation to the number of actions per slot. The simulation parameters are: $K=2$, $d_k \in $\{2, 10, 100, 1000\}, $T=50$, $N=1e7$. The plot shows the percentage relative improvement achieved by the proposed \newpi estimator over the PI estimator.}
  \label{fig:D1D2}
\end{center}
\end{figure}

\begin{figure*}
\begin{center}
  \includegraphics[width=3.15in]{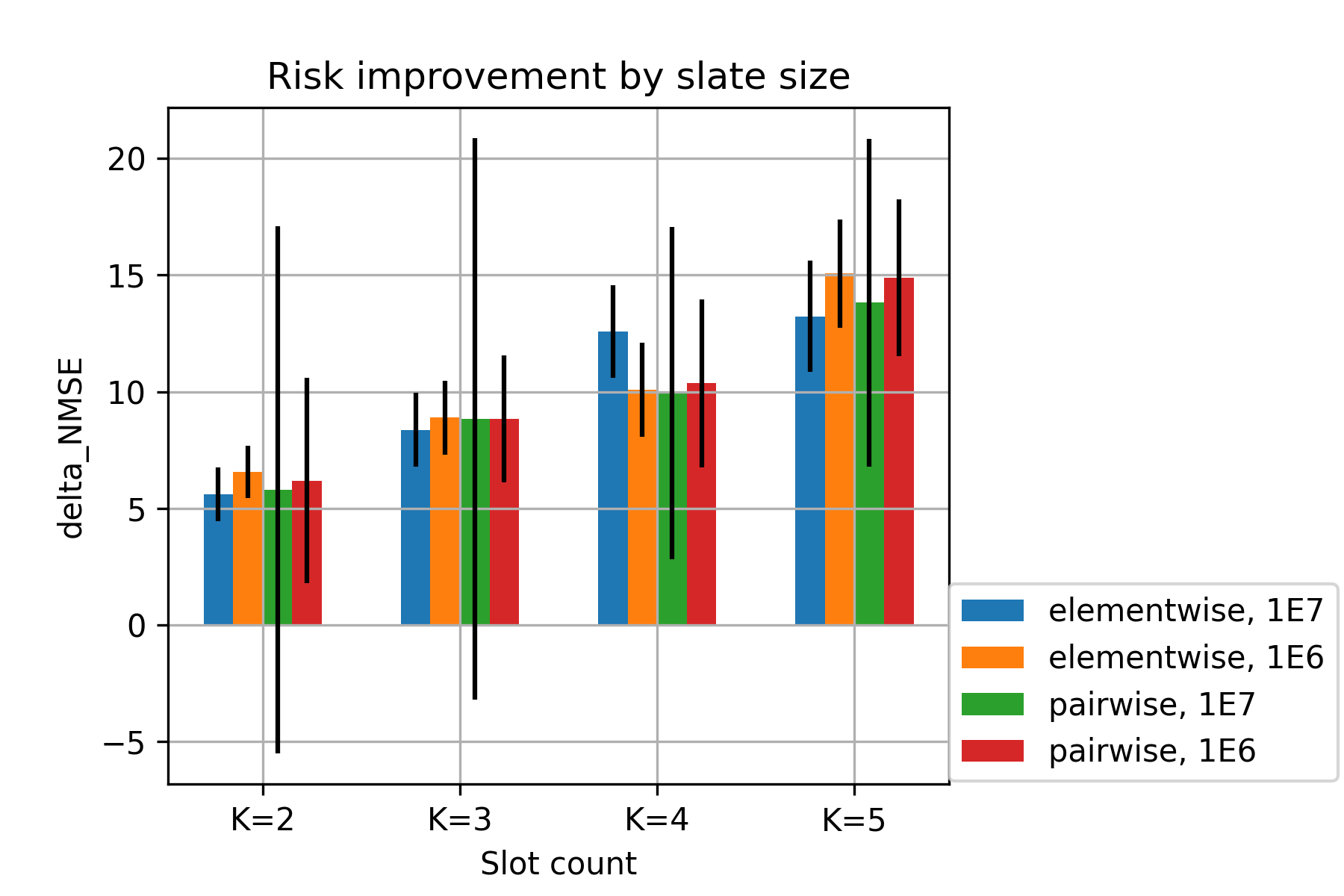}
  \includegraphics[width=3.15in]{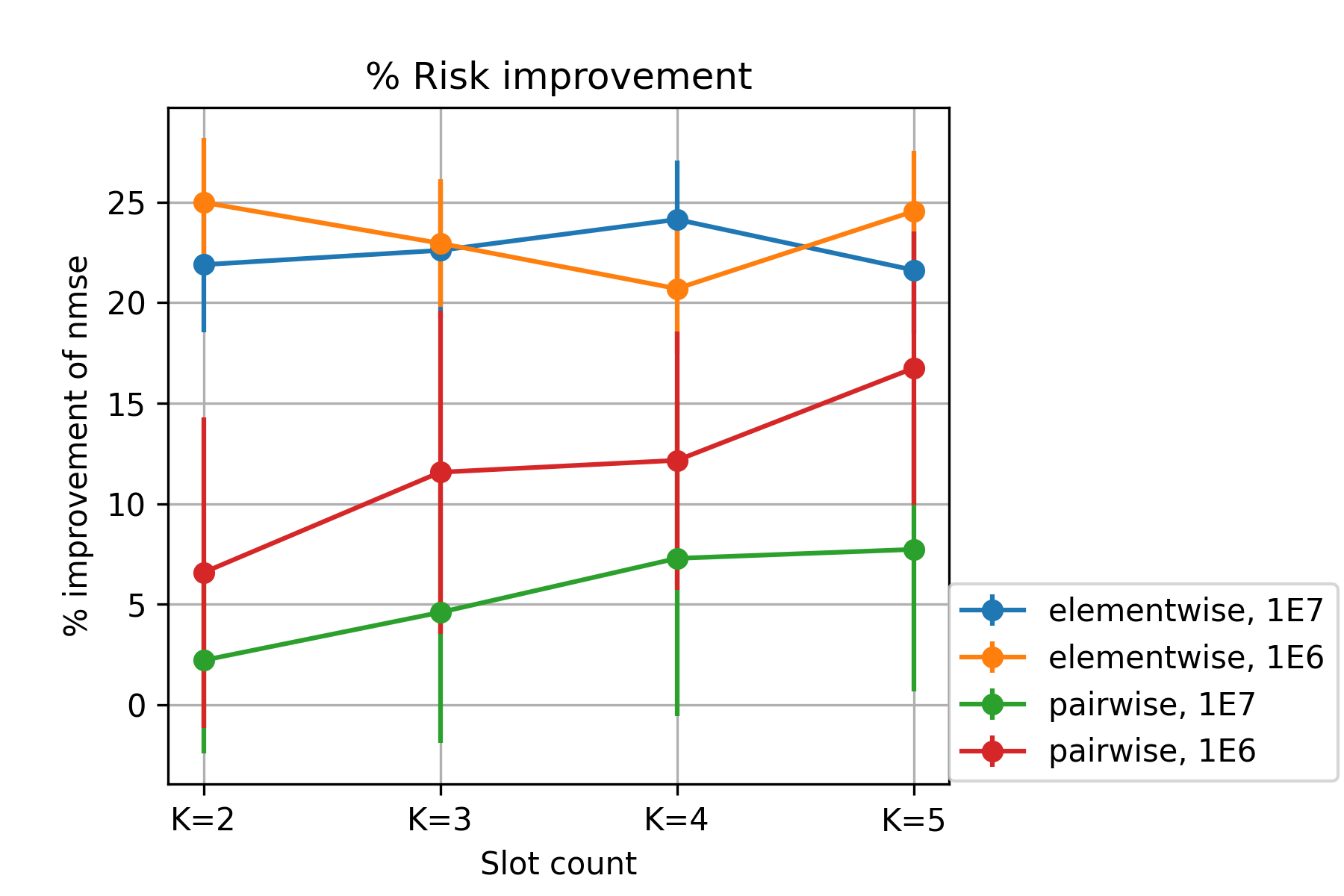}
  \caption{Behavior of the estimators under elementwise and pairwise additive reward models, as a function of the number of slots $K \in \{2,3,4,5\}$, and sample size $N \in \{1e6, 1e7\}$. Here $T=200$ and the number of simulations per tensor is 500. Slot action cardinalities, $D$, are set deterministically. Please see the text for details. The plot on the left shows the reduction in risk achieved by the \newpi estimator w.r.t the PI estimator. The plot on the right shows relative risk improvement over the PI estimator. }
  \label{fig:linearK}
\end{center}
\end{figure*}

\subsection{Number of slots}
\label{sec:pairwise}

In this experiment we investigate the performance of \newpi for slates of different sizes. We first simulate with a deterministic scheme for slot action cardinalities. Using $\min(D) =2$ and $\max(D) = 100$, we evenly divide this interval to obtain the cardinalities for the slots in the slate. For example, for a slate with $K=4$,  we pick the action cardinalities for the slots as $D = [2, 33, 66, 100]$. $P'$ is set to 0.25 and the true mean Bernoulli rate in the simulator is also set to $\bar P = 0.25$. We study the behavior under two distinct reward models: The elementwise additive model \eqref{eq:Psum} (for which both \newpi and PI are unbiased) and the pairwise additive reward model \eqref{pairwise}. Here we set $T=200$ and reduce the number of simulations per tensor to 500. Two different sample sizes are investigated, $N \in \{1e6, 1e7\}$. The estimators are examined for $K \in \{2,3,4,5\}$. We measure the improvement by the metric $delta\_nmse = N*MSE_{PI}^{} - N*MSE_{\newpi}^{}$ for every tensor and report the improvements in the left plot of Figure~\ref{fig:linearK}. From this plot we see that in both settings the risk improvement of \newpi over PI grows linearly with the number of slots, as predicted by eq.~\eqref{mainresult}. The plot on the right shows percentage improvement over the PI estimator.

We next repeat the experiment, but this time assigning the slot action cardinalities randomly from the interval, i.e., $d_k \sim U(2,100)$. We only investigated the elementwise additive reward model. The rest of the setting remains unchanged. This simulation helps us tie together the impact of the number of slots and the slot action cardinalities in one coherent picture. The results are shown in Figure \ref{fig:KandD}. The results closely follow the predictions from the theoretical analysis, see eq.~\eqref{eqMH}. In particular, the observed linear dependence of the improvement in risk ($delta\_nmse$) with K, as well as with the gap between the arithmetic and harmonic means of the slot action cardinality sets, are in close agreement with the theory, thus validating the robust gains of the \newpi estimator that were predicted by our analysis.

\begin{figure}
\begin{center}
  \includegraphics[width=3.45in]{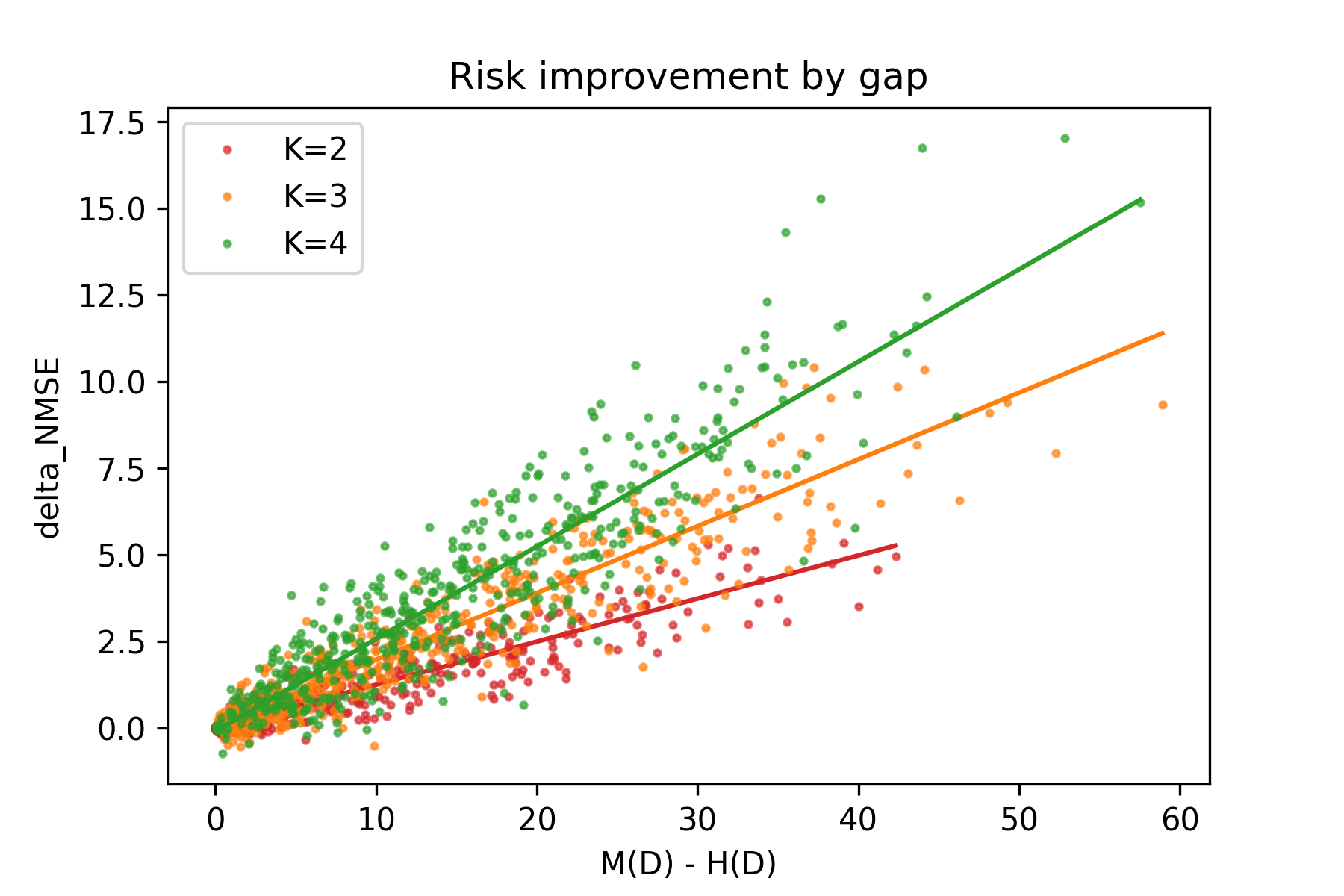}
  \caption{A combined view of the impact of the number of slots and the slot-level action counts on risk improvement. The reward model is elementwise additive. 
  $K \in \{2,3,4,5\}$ and 
  $N \in \{1e6, 1e7\}$.  
  We set
  $T=200$, the number of simulations per tensor is 500, and $d_k \sim U(2,100)$. Each point corresponds to $delta\_{nmse}$ for a tensor. We also show the best-fit line through the point cloud corresponding to each value of K. The $R^2$ values for $K = [2, 3, 4]$ are $[0.93, 0.93, 0.91]$ respectively.}
  \label{fig:KandD}
\end{center}
\end{figure}

\section{Conclusions}
We studied the problem of off-policy evaluation for slate bandits through the lens of Bayes risk. Using a control variate approach, we identified a new estimator that is guaranteed to have lower risk than the PI estimator of  \citet{swaminathan2017nips}. We provided theoretical and empirical evidence that the risk improvement over PI grows linearly with the number of slots and linearly with the gap between the arithmetic and the harmonic mean of a set of slot-level $\chi^2$-divergences between the logging and target policies. 

An interesting avenue for further research would be to address more complex reward models and analyze the involved bias-variance tradeoffs under model misspecification. Another avenue would be to relax the requirement of unbiasedness of the resulting estimator, and directly optimize the risk within the parametric class of estimators, similar to 
\citet{farajtabar2018icml} and \citet{su2020doubly}. 


\end{document}